\documentclass{article}

\usepackage[T1]{fontenc}
\usepackage[english]{babel}
\usepackage{csquotes}
\newcommand{\Arxiv}[2]{#2}

\newcommand{\authorrunning}[1]{}
\newcommand{\institute}[1]{}
\newcommand{\orcidID}[1]{${}^{\textnormal{\href{https://orcid.org/#1}{[#1]}}}$}
\newcommand{\keywords}[1]{\\\;\\\textbf{Keywords: } \begingroup\newcommand{\and}{$\cdot$ }#1\endgroup}

\usepackage{amsthm}
\theoremstyle{definition}
\newtheorem{definition}{Definition}

\theoremstyle{plain}
\newtheorem{lemma}{Lemma}
\newtheorem{proposition}{Proposition}
\newtheorem{theorem}{Theorem}
\newtheorem{corollary}{Corollary}

\theoremstyle{remark}

\usepackage{graphicx}

\usepackage{amsfonts}
\usepackage{amsmath}
\usepackage{amssymb}
\usepackage{subcaption}
\usepackage{hyperref}
\usepackage[capitalize]{cleveref}
\usepackage{algorithm}
\usepackage{algpseudocode}
\usepackage{booktabs}

\usepackage{tikz}
\usetikzlibrary{positioning}

\usepackage{todonotes}

\newcommand{\R}{\mathbb{R}}
\newcommand{\N}{\mathbb{N}}

\newcommand{\X}{\mathcal{X}}
\newcommand{\T}{\mathcal{T}}
\newcommand{\D}{\mathcal{D}}
\newcommand{\PT}{P_T}
\newcommand{\WDS}{\textnormal{WDS}}
\newcommand{\DP}{\textnormal{DP}_{\textnormal{rph}}}
\newcommand{\WDSfh}{\WDS_\textnormal{fh}}
\newcommand{\holist}{H}
\renewcommand{\d}{\,\mathrm{d}}
\newcommand{\Null}{\mathcal{N}}
\newcommand{\W}{\mathcal{W}}
\newcommand{\proj}{\textnormal{proj}}
\newcommand{\1}{\mathbf{1}}
\newcommand{\Rm}{\mathbf{R}}
\newcommand{\Dm}{\mathbf{D}}
\newcommand{\Pm}{\mathbf{P}}
\newcommand{\Wm}{\mathbf{W}}
\newcommand{\diag}{\textnormal{diag}}

\newcommand\blfootnote[1]{%
  \begingroup
  \renewcommand\thefootnote{}\footnote{#1}%
  \addtocounter{footnote}{-1}%
  \endgroup
}

\begin{document}
\title{An Algorithm-Centered Approach \\ To Model Streaming Data\thanks{Funding in the scope of the BMBF project KI Akademie OWL under grant agreement No 01IS24057A and 
MKW NRW project SAIL under grant agreement No NW21-059B is gratefully acknowledged.}\Arxiv{}{\:\:\footnote{This manuscript is currently under review at the Symposium on Intelligent Data Analysis (IDA 2025).}}}

\author{Fabian Hinder\orcidID{0000-0002-1199-4085}${}^{,\ddagger}$\and 
Valerie Vaquet\orcidID{0000-0001-7659-857X}\and 
David Komnick\and
Barbara Hammer\orcidID{0000-0002-0935-5591}%
\Arxiv{}{ \\\;\\
Bielefeld University, Bielefeld, Germany \\
\texttt{\{fhinder,vvaquet,bhammer\}@techfak.uni-bielefeld.de}}
}
\authorrunning{F. Hinder et al.}
\institute{Bielefeld University, Bielefeld, Germany \\
\email{\{fhinder,vvaquet,bhammer\}@techfak.uni-bielefeld.de}}
\maketitle              
\blfootnote{\!\!\!\!${}^\ddagger$ Corresponding Author}
\begin{abstract}
Besides the classical offline setup of machine learning, stream learning constitutes a well-established setup where data arrives over time in potentially non-stationary environments. Concept drift, the phenomenon that the underlying distribution changes over time poses a significant challenge. Yet, despite high practical relevance, there is little to no foundational theory for learning in the drifting setup comparable to classical statistical learning theory in the offline setting. This can be attributed to the lack of an underlying object comparable to a probability distribution as in the classical setup. While there exist approaches to transfer ideas to the streaming setup, these start from a data perspective rather than an algorithmic one. In this work, we suggest a new model of data over time that is aimed at the algorithm's perspective. Instead of defining the setup using time points, we utilize a window-based approach that resembles the inner workings of most stream learning algorithms. We compare our framework to others from the literature on a theoretical basis, showing that in many cases both model the same situation. Furthermore, we perform a numerical evaluation and showcase an application in the domain of critical infrastructure. 
\end{abstract}

\section{Introduction}
\label{sec:intro}
Most works on machine learning consider the batch setup where data is drawn i.i.d. from a single distribution. However, in many real-world applications, data is generated over time and possibly subject to various changes which are commonly referred to as \emph{concept drift}~\cite{oneortwo,lu2018,gama} or drift for shorthand. Drift can arise from a variety of sources as, for instance, seasonal patterns, shifting user demands, sensor aging, and environmental factors.

These frequently unforeseen distributional changes can yield substantial drops in the performance of machine learning models that often require retraining or even the redesign of entire pipelines. Thus, a considerable body of work focuses on the \emph{stream learning} setup~\cite{lu2018,gama} and designs models that are adaptable to distributional changes in the data arriving as an infinite sequence known as a \emph{data stream}. A common strategy in this field is to use sliding window-based approaches: By updating the model using only the most recent samples while discarding older ones, the model adapts to the current distribution~\cite{gama}.

In contrast to practical approaches, theory in the stream learning setup is still limited. As standard frameworks, for example, statistical learning theory~\cite{shalev2014understanding}, impose the assumption that all considered data is drawn i.i.d. from a single distribution, drift complicates theoretical analysis~\cite{ida2023}. Besides, the established way of modeling data over time and defining drift are not suitable for analyzing stream machine learning on a theoretical level~\cite{dawidd,oneortwo,ida2023}. Common definitions focus on a time-point-wise perspective~\cite{dawidd,gama}, assigning a distribution per time point rather than over a period. Yet, as described before, most algorithms employ time-window-based approaches, looking at data over time spans~\cite{gama}. This raises questions about the validity of point-wise drift definitions for practical considerations.

To address these challenges, we propose a new framework to model data over time. It is centered around the idea of time windows rather than time points. Thus, our framework aligns directly with the algorithmic perspective and provides a more practical basis for defining drift. We compare our approach to traditional time point definition and investigate the mathematical relationships between the two. Interestingly, despite their distinct starting points, both frameworks capture similar information. Thus, our analysis not only offers deeper theoretical insights but also justifies previously used approaches. In particular, it supports the \emph{time-as-a-feature} paradigm~\cite{dawidd} that has led to the development of several recent algorithms~\cite{oneortwo}.

This paper is structured as follows: First (\cref{sec:preliminary}) we recall the basics of stream learning and concept drift. In the second part (\cref{sec:main}) we introduce our new formal framework (\cref{sec:WDS}), and compare it to the existing definitions of drift theoretically (\cref{sec:comparison}) and algorithmically (\cref{sec:algo}). Afterward, we perform two experiments (\cref{sec:Exp1,sec:Exp2}), showcasing both the found connection between the window-based and time-point-based setup as well as a practical relevant use case of our ideas in the domain of critical infrastructure, before we conclude the paper (\cref{sec:conclusion}).

\section{Preliminary Notes: Stream Learning and Concept Drift}
\label{sec:preliminary}

In the classical batch setup of machine learning one considers a dataset $S = (X_1,\dots,X_n)$ with $X_i \sim \D$ are i.i.d.~\cite{shalev2014understanding}. In \emph{stream learning} this setup is extended to an infinite sequence of observations $S = (X_1,X_2,\dots)$ with the single data points $X_i$ arriving over time. To deal with this, most stream learning algorithms are based on \emph{sliding windows} which typically contain the last $n$ data points, i.e., $S_{i-n:i} = (X_{i-n},\dots,X_i)$~\cite{lu2018}. In this case, data analysis and model training are performed based on $S_{i-n:i}$ replacing the dataset in the batch setup. In particular, $S_{i-n:i}$ is updated once the next data point $X_{i+1}$ arrives~\cite{gama,lu2018}.

As the data is collected over longer periods of time, the underlying distribution might change -- a phenomenon known as \emph{concept drift}~\cite{gama,lu2018,oneortwo}. There are different approaches to model concept drift. 
The most common way is to allow each data point to follow a (potentially different) distribution $X_i \sim \D_i$. Drift is then defined as a change of the data-point-wise distribution, i.e., $\exists i,j : \D_i \neq \D_j$~\cite{gama,lu2018}. Thus, finite streams without drift correspond to the classical i.i.d. setup. 

However, as pointed out by \cite{dawidd,ida2023} this setup is not well suited to formulating learning theory results due to its lack of an underlying reference object like a distribution in classical statistics. To solve this issue \cite{ida2023,dawidd} suggested explicitly considering time by introducing a new random variable $T_i$ marking the arrival time of $X_i$. In this fully statistical setup, the data point's distribution is determined by the arrival time, i.e., $X_i \mid T_i = t \sim \D_t$. Drift is defined as the probability of observing two different distributions being larger than zero. This leads to:

\begin{definition}[Distribution Process, Time Window, Holistic Distribution~\cite{oneortwo}]
\label{def:DP}
Let $(\X,\Sigma_\X)$ and $(\T,\Sigma_\T)$ be measure spaces. A \emph{(regular post hoc) distribution process} $(\D_t,\PT)$ on $\X$ over $\T$ is a Markov kernel $\D_t$ from $\T$ to $\X$, i.e., for all $t \in \T$ $\D_t$ is a distribution on $\X$ such that $t \mapsto \D_t(A)$ is measurable for all $A \in \Sigma_\X$, and $\PT$ a distribution on $\T$.

A distribution process has \emph{drift} if the chance of observing two different distributions is larger than 0, i.e., $\exists A \in \Sigma_\X : \PT^2(\{(s,t) \in \T \mid \D_t(A) \neq \D_s(A)\}) > 0$.

We call a set $W \in \Sigma_\T$ with $\PT(W) > 0$ a \emph{time window}. We refer to the (uniquely determined) distribution on $\X \times \T$ such that $\D(A \times W) = \int_W \D_t(A) \d \PT$ as the \emph{holistic distribution} and to the conditional distribution $\D_W(A) = \linebreak\D(A \times W \mid \X \times W)$ as the \emph{window mean distribution} for a time window $W$.

We denote the set of all (regular post hoc) distribution processes as $\DP(\X,\T)$ and the map taking a distribution process to its holistic distributing as\linebreak $\holist : \DP(\X,\T) \to \Pr(\X \times \T)$.
\end{definition}

Due to the fully statistical nature of this setup, it is possible to draw new data points from a distribution process.
Furthermore, as pointed out by \cite{oneortwo,ida2023}, a crucial advantage of this model is that it allows us to derive a window-wise distribution in a formally sound way. Since windows play a fundamental role in stream learning algorithms~\cite{lu2018}, an approach corresponding to windows is better suited to transfer theoretical ideas like statistical learning theory to the streaming setup~\cite{ida2023}.

In this case one considers the time-window-based samples, i.e., for a stream $S = ((X_1,T_1),(X_2,T_2),\dots)$ and $W \in \Sigma_\T$ we have the \emph{window sample}
\begin{align*}
    S_W = (X_i)_{i : T_i \in W}.
\end{align*}
As the data distribution depends on time, taking the limiting distribution relates to taking the sampling frequency to infinite~\cite{oneortwo} so that for increasing frequencies $n$ the size of $S_W^{(n)}$ extracted from the stream $S^{(n)}$ tends to infinity. 

If $S$ was initially obtained as an i.i.d. sample of the holistic distribution $\D$ then $S_W$ is an i.i.d. sample of $\D_W$~\cite{dawidd}. However, in practice, different ways of obtaining samples can be used. This raises the question of whether an algorithm-centered point of view should be favored over a time-point-centered one. 

In the following, we will provide a new way to model data over time that primarily targets the distribution on windows and thus a more algorithmic point of view. We will then compare the resulting formal model to the time-point-wise point of view discussed above.

\section{A Stream Learning Inspired Framework}
\label{sec:main}

As discussed in the last section, basically all algorithmic approaches to stream learning are based on the idea of sliding windows. On the other hand, most formal models used to describe data over time consider the data for each time point separately. As pointed out by \cite{dawidd} transitioning from the time-point-wise point of view to the window-based point of view is not only crucial for understanding stream learning but also poses a challenge as it is in general not clear how to compute these window distributions in the different formal models. 

To bridge this gap we will now introduce a new formal model that, in contrast to most formal models used to describe stream setups, is not based on the time-point-based point of view but the more algorithmic window-based point of view.

\subsection{A Formal Model of Window-based Distributions}
\label{sec:WDS}

To construct a window-based formal model for stream learning we need to specify two things: 1)~the windows suited for analysis, and 2)~the distributions thereon. For simplicity, we will assume that a time window is a set of time points which is not necessarily an interval. The main modeling thus focuses on the question of which sets of time points are reasonable windows and how the distributions placed on those need to relate to each other. We will address the concept of windows in \cref{def:WS} and the concept of the distributions on those windows in \cref{def:WDS}. 

To model the \emph{window system (WS)} we will not focus on a specific algorithmic implementation but the more general question of which windows constitute a reasonable choice in the sense that they are ``large enough'' to admit estimates when we take the sampling frequency to infinity. For independent $X_i$ this is the case if $|S_W^{(n)}| \to \infty$ as $n \to \infty$.
We obtain the following definition, which to formally capture this idea also requires considering \emph{null windows} that are not suited for estimates:
\begin{definition}[Window System (WS)]
\label{def:WS}
Let $(\T,\Sigma_\T)$ be a measure space. A \emph{window system on $\T$} $(\W,\Null)$ is a set system $\W \subset\Sigma_\T$ of \emph{windows} and a $\sigma$-ideal $\Null \subset \Sigma_\T$ of \emph{null windows} such that the following hold
\begin{enumerate}
    \item $\W$ is closed under finite disjoint unions
    \item $\W$ and $\Null$ are disjoint and $\W \cup \Null$ is a semi-ring 
    \item $\W$ locally generates $\Sigma_\T$, i.e., $(\Sigma_\T)_{|W} \subset \sigma(\W \cup \Null)$ for all $W \in \W$
    \item $\W$ covers $\T$, i.e., for $S \in \Sigma_\T$ if $S \cap W \in \Null$ for all $W \in \W$ then $S \in \Null$.
\end{enumerate}

If there is a $W \in \W$ such that $\T \setminus W \in \Null$ we say that the window system has a \emph{finite horizon}.
\end{definition}

We shortly explain the axioms: the overarching idea is that $\W$ captures sets which allow a reasonable statistical estimate of the properties of the underlying process. 
Axiom 1 captures that if we have two samples which are large enough for that, then so is their union.
Axiom 2 models the fact that there are sets in which there are no observations which can be obtained e.g. by intersecting larger samples.
Axioms 3 and 4 assure that the windows encode just as much information as the $\sigma$-algebra.

The term \emph{finite horizon} relates to the idea that we can, in theory, keep the entire stream in memory. It is therefore a way to limit our considerations to streams of finite time spans.

Based on this concept define \emph{windowed distributions systems (WDS)} that capture the observed windows underlying probabilities and their consistency. For each window, we define the limiting distribution from which the sample $S_W$ is drawn. The concept of null windows is included by requiring that two windows that only differ by a null window have the same limiting distribution as they produce the essentially same samples, resulting in:

\begin{definition}[Windowed Distribution System (WDS), Constant WDS, Extension]
\label{def:WDS}
Let $(\W,\Null)$ be a window system and $(\X,\Sigma_\X)$ a measure space. A \emph{Windowed Distribution System (WDS)} $(\W,\Null,\D_W)$ is a collection of probability measures $(\D_W)_{W \in \W}$ on $\X$ such that
\begin{enumerate}
    \item for $W_1,W_2 \in \W$ with $W_1 \vartriangle W_2 \in \Null$ we have $\D_{W_1} = \D_{W_2}$, 
    \item if $W_1 \subset W_2 \subset \dots \nearrow W$ with $W,W_1,W_2,\dots \in \W$ then $\lim_{n \to \infty} \D_{W_n} = \D_{W}$,
    \item \label{def:WDS:comp} for pairwise disjoint windows $W_1,\dots,W_n \in \W$ there are numbers\linebreak $0 < \lambda_1,\dots,\lambda_n$ such that for all $I \subset \{1,\dots,n\}$ we have 
    \begin{align*}
    	\D_{W_I} = \frac{\sum_{i \in I} \lambda_i \D_{W_i}}{\sum_{i \in I} \lambda_i} \qquad \text{where} \qquad W_I := \bigcup_{i \in I} W_i.
    \end{align*}
\end{enumerate}

If for all $W_1,W_2 \in \W$ we have $\D_{W_1} = \D_{W_2}$ we say that the WDS is \emph{constant}.
A WDS $(\bar\W,\bar\Null,\bar\D_W)$ (on the same base spaces) is called an \emph{extension} of $(\W,\Null,\D_W)$ iff $\W \subset \bar\W$, $\Null \subset \bar\Null$, and $\D_W = \bar\D_W$ for all $W \in \W$.

We denote the set of all WDS as $\WDS(\X,\T)$ and those with finite horizons as $\WDSfh(\X,\T)$.
\end{definition}

The axioms capture that the null windows do not impact the probabilities (axiom 1), probabilities are compatible with sampling (axiom 2) and compatible with unions (axiom 3). 
Axioms 1 and 3 are motivated by finite samples: if two estimates are made based on the same samples, then the estimated value is the same (axiom 1), if we combine two samples then the combined estimate is a convex combination of the single estimates (axiom 3).

In the next section, we will show that the concept of WDSs is compatible with the time-point-wise model of a distribution process but more general. 

\subsection{Connection to the Time-Point-Wise Setup}
\label{sec:comparison}

\begin{figure}[t]
    \centering
    \resizebox{\textwidth}{!}{
    {
    \begin{tikzpicture}[
nd/.style={rectangle, minimum width=2cm, minimum height=1cm}
]
    \node[nd] (DP)                         {$\underset{\textnormal{r.p.h. Distribution} \atop \textnormal{Process (Def.~\ref{def:WDS} \cite{dawidd})}}{(\D_t,\PT)}$};
    \node[nd] (HD)    [right = of DP]      {$\underset{\textnormal{Holistic}\atop\textnormal{Distribution \cite{dawidd}}}{\D}$};
    \node[nd] (help1) [right = of HD]      {};
    \node[nd] (WDS)   [above = of help1]   {$\underset{\textnormal{WDS (Def.~\ref{def:WDS})}\atop\;}{\D_W}$};
    \node[nd] (WDSPT) [below = of help1]   {$\underset{\textnormal{WDS + Comp.}\atop\textnormal{Temp.Dist (Def.~\ref{def:comp_PT})}}{(\D_W,\PT)}$};
    
    \draw[->] (DP.north east) to [bend right=-45] node[midway,above]{Fubini} (HD.north west);
    \draw[<-] (DP.south east) to [bend right=45] (HD.south west);
    
    \draw[->] (HD.north east) --node[midway,above,sloped]{Prop.~\ref{prop:induced_WDS}} (WDS.south west);
    \draw[->] (WDS.south) --node[midway,above,sloped]{Prop.~\ref{prop:exist_PT}} (WDSPT.north);
    \draw[->] (WDSPT.north west) --node[midway,below,sloped]{Prop.~\ref{prop:premain}} (HD.south east);
    
    \draw[->, draw=gray] (WDSPT.west) -|node[above]{\hspace{4cm}\cref{thm:main}} (DP.south);
    \draw[->, draw=gray] (DP.north) |- (WDS.west);
    
    \end{tikzpicture}\hspace{2cm}}}
    \caption{Overview of considered theorems, setups, and stages. }
    \label{fig:overview}
\end{figure}
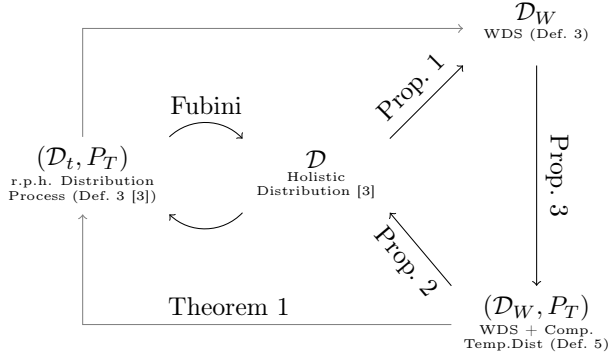

In the last section, we introduced WDSs as an alternative model to describe distributions in data streams. In contrast to distribution processes which are defined time-point-wise and consider the problem from a data modeling perspective, WDSs are based on the idea of windows which form the basis for many stream learning algorithms. 
In the following, we will consider the question of how these concepts are related, see \cref{fig:overview}. 

Given a distribution process we can construct a WDS using the window mean distributions, as already suggested in \cref{sec:preliminary}:

\begin{proposition}
\label{prop:induced_WDS}
By taking the windows, null windows, and window mean distributions of a holistic distribution we obtain a WDS on finite horizons, i.e., 
\begin{align*}
\begin{array}{rl}
i : \Pr(\X \times \T) &\to \WDSfh(\X,\T) \\
    \D &\mapsto (\W,\Null,\D_W) 
\end{array}
\quad\text{where}\quad
\begin{array}{rl}
\W &= \{S \in \Sigma_\T \mid \D(\X \times S) > 0\} \\
\Null &= \{S \in \Sigma_\T \mid \D(\X \times S) = 0\} \\
\D_W(A) &= \D(A \times W \mid \X \times W)
\end{array}
\end{align*}
is a well-defined map. 
It extends via holistic distributions to distribution processes
\begin{align*}
    I : \DP(\X,\T) &\to \WDSfh(\X,\T)
\end{align*}
Furthermore, $I(\D_t)$ is not constant if and only if $\D_t$ has drift.
\end{proposition}
\begin{proof}
\Arxiv{All proofs can be found in the ArXiv version \todo{cite}}{All proofs are given in the appendix}
\end{proof}

\cref{prop:induced_WDS} shows that every distribution process induces a WDS. This construction can be extended to non-probability measures showing that WDSs can model more general situations. This however is beyond the scope of this paper.

\begin{definition}
For a distribution process $\D_t$ on $\X$ over $\T$ or a probability distribution $\D$ on $\X \times \T$ we refer to $I(\D_t)$ and $i(\D)$, respectively, as the \emph{induced WDS}, where $i$ and $I$ are the maps from \cref{prop:induced_WDS}.
\end{definition}

In the following, we address the opposite direction considered in \cref{prop:induced_WDS}. As before we will do this constructively. We proceed in two steps: 1)~reconstruct the holistic distribution, given the WDS and $\PT$, and then 2)~reconstruct $\PT$ from the WDS.
To proceed with 1) we have the following:

\begin{proposition}
\label{prop:premain}
Given the time marginal $\PT(W) = \D(\X \times W)$ we can reconstruct the holistic distribution $\D$ from the induced WDS $i(\D)$, i.e., the map
\begin{align*}
    i \times \proj_{\Pr(\T)} : \Pr(\X \times \T) &\to \WDSfh(\X,\T) \times \Pr(\T) \\
    \D &\mapsto (i(\D),\; \PT)
\end{align*}
has a left inverse $h$. Furthermore, right concatenation results in extensions, i.e., for a WDS $\D_W$ with compatible $\PT$ we have that $i \circ h (\D_W,\PT)$ extends $\D_W$.
\end{proposition}

Here, we considered the case where we pair the induced WDS directly with the temporal distribution $\PT$. Yet, not every temporal distribution fits with every WDS as the weights $\lambda_i$ in condition~\ref{def:WDS:comp} of \cref{def:WDS} have to match the measure. This can be formulated as follows:

\begin{definition}[Compatible Time Distribution]
\label{def:comp_PT}
Let $(\W,\Null,\D_W)$ be a WDS. We say that a measure $\PT$ on $\T$ is \emph{compatible} with $\D_W$ iff the following hold
\begin{enumerate}
    \item for all $W \in \W$ we have $0 < \PT(W) < \infty$ 
    \item for all $N \in \Null$ we have $\PT(N) = 0$ 
    \item for all disjoint $W_1,\dots,W_n \in \W$ it holds
    \begin{align*}
        \D_W = \frac{\sum_{i = 1}^n \PT(W_i) \D_{W_i}}{\sum_{i = 1}^n \PT(W_i)}
    \end{align*}
    where $W = W_1 \cup \dots \cup W_n$.
\end{enumerate}
\end{definition}

The next step is to construct a compatible time distribution from the WDS. 
If the WDS is constant then it does not provide enough structure to construct the measure as the weights can be chosen arbitrarily. This not only leads to ambiguity regarding the choice of $\PT$ but it might even be the case that no such measure exists at all. 
In contrast, for non-constant WDS, the weights allow to construct $\PT$ from the WDS. As the weights only give a relative comparison we can only hope to reconstruct it up to scaling, but aside from that, it is unique.

\begin{proposition}
\label{prop:exist_PT}
For every non-constant WDS on finite horizons, there is a unique compatible probability time distribution, i.e., there is a map
\begin{align*}
    R : \left\{\D_W \in \WDSfh(\X,\T) \:\left|\: \D_W \text{ is not} \atop \text{constant}\right.\right\} \to \Pr(\T),
\end{align*}
such that $R(\D_W)$ is compatible with $\D_W$.
\end{proposition}

Combining \cref{prop:premain,prop:exist_PT} we obtain the following consequence:

\begin{corollary}
\label{cor:premain}
Holistic distributions with drift\footnote{As shown in \cite{dawidd} a distribution process has no drift if and only if $\D = \D_\T \times \PT$} model the same situations as non-constant WDSs on finite horizons, i.e., the map 
\begin{align*}
    i_{|\text{drift}} : \Pr(\X \times \T) \setminus \Pr(\X) \times \Pr(\T) &\to \left\{\D_W \in \WDSfh(\X,\T) \:\left|\: \D_W \text{ is not} \atop \text{constant}\right.\right\}
\end{align*}
has a left inverse $f$ and it holds $f = h \circ (\textnormal{id}_{\WDS(\X,\T)} \times R) \circ \Delta$ with $\Delta$ the diagonal, $h$ and $R$ are the maps from \cref{prop:premain,prop:exist_PT}.
\end{corollary}

In \cref{prop:premain} and \cref{cor:premain} we considered probability distributions on $\X \times \T$ which we interpreted as the holistic distributions of some distribution process. Although the holistic distribution essentially encodes the entire practically relevant information of the distribution process~\cite{oneortwo} it is in general not possible to uniquely reconstruct the Markov kernel from the holistic distribution. However, assuming that data space and time domain are sufficiently simple this is possible. In this case, we can rephrase \cref{prop:premain} and \cref{cor:premain} in terms of distribution processes:

\begin{theorem}
\label{thm:main}
Assuming $\X$ and $\T$ are standard Borel spaces, e.g., $\N,[0,1],\R,\R^d$, then the results of \cref{prop:premain} and \cref{cor:premain} translate to distribution processes, i.e., the maps
\begin{align*}
    I \times \proj_{\PT} : \DP(\X,\T) &\to \WDSfh(\X,\T) \times \Pr(\T) \\
    I_{|\text{drift}} : \left\{\D_t \in \DP(\X,\T) \:\left|\: \D_t \text{ has} \atop \text{drift}\right.\right\} &\to \left\{\D_W \in \WDSfh(\X,\T) \:\left|\: \D_W \text{ is not} \atop \text{constant}\right.\right\}
\end{align*}
have left inverses and right concatenation result in extensions (cf.~\cref{prop:premain}). 
\end{theorem}

As can be seen, for finite temporal horizons, in the interesting cases, i.e., those where there is drift, distribution processes, and WDSs model the same situations although the representations are different. This is particularly relevant as distribution processes are easier to work with for modeling tasks while WDSs offer a more algorithm-centered point of view. Thus, \cref{thm:main} allows us to connect the data modeling and the algorithm design and analysis perspective.

Indeed, several algorithms are based on the idea of treating time more or less explicitly as a random variable~\cite{dawidd,oneortwo,ida2023}; \cref{thm:main} shows that this is not just a modeling choice but, as long as we process the data using sliding windows, deeply build into the algorithmic design independent of whether or not considered explicitly.

Furthermore, if we drop the assumption of finite temporal horizons, WDSs model more general situations that can no longer be described using (post hoc) distribution processes. This, however, is beyond the scope of this paper. 

For all statements above we give constructive proofs. In the next section, we will translate those to approximate algorithms for finite data, before turning to an empirical analysis in \cref{sec:exp}.

\section{An Algorithmic Approach}
\label{sec:algo}

In the last section, we discussed the connection between WDSs and distribution processes. While deriving a WDS from a distribution process is straightforward (\cref{prop:induced_WDS}), the other direction is more complicated (\cref{prop:exist_PT,prop:premain}, and \cref{thm:main}). However, while our considerations assure the existence of a distribution process the applied techniques do not directly translate to algorithms. 
In the following, we consider the algorithmic task: given a sequence of windows $W_1,...,W_n \subset \R$ and probabilities $\D_{W_i}(A_j)$ find the associated distribution process, i.e., $\PT$ and $\D_t$ such that 
\begin{align}
    \frac{1}{\PT(W_i)}\int_{W_i} \D_t(A_j) \d \PT(t) \approx \D_{W_i}(A_j).  \label{eq:Exp1:problem1}
\end{align}

\begin{algorithm}[!t]
	\caption{Solver for \eqref{eq:Exp1:opt1}}
	\label{algo:opt}
	\begin{algorithmic}[1]
		\Function{ReconstructDistributionProcess}{$\Wm$, $\Rm$} 
		\State $\Pm \gets \textsc{Normalize}(\1)$
		\Repeat 
        \State $\Dm \gets \textsc{Normalize}(\textsc{NNLS}(\Wm \diag(\Pm) \Dm , \diag(\Wm \Pm)\Rm))$ \label{algo:opt:compteD}
        \State $S \gets \textsc{VStack}(\;(\Wm_{it}\Dm_{tj})_{(i,j),t} - (\Wm_{it}\Rm_{ij})_{(i,j),t}, \1)$
        \State $\Pm \gets \textsc{Normalize}(\textsc{NNLS}(S,(0,\dots,0,1)^\top)$
		\Until{Convergence}
		\State \Return $\Pm,\Dm$
		\EndFunction
	\end{algorithmic}
\end{algorithm}
Due to limitations regarding the temporal resolution we will consider the smallest set that allows to describe each $W_i$ as a disjoint union as time, i.e., $\T = \{W_1^{s_1} \cap \dots \cap W_n^{s_n} \mid s_i \in \{-1,1\}\}$ with $W_i^1 = W_i, W_i^{-1} = W_i^C$. Furthermore, we assume that the sets $A_1,\dots,A_m$ are disjoint and cover $\X$. This allows us to conveniently write $\Rm_{ij} = \D_{W_i}(A_j)$ as a $n \times m$-matrix, $\Dm_{tj} = \D_t(A_j)$ as a $N \times m$-matrix, and $\Pm_t = \PT(\{t\})$ as a $N$-dimensional vector, were $N = |\T| \leq 2^n$. By representing $W_i$ by its elementary windows we can represent the windows $W_1,\dots,W_n$ as yet another $n \times N$ matrix $\Wm$ allowing us to express \cref{eq:Exp1:problem1} as
\begin{align}
    \diag(\Wm \Pm)^{-1} \Wm \diag(\Pm) \Dm \approx \Rm \label{eq:Exp1:problem3}
\end{align}
where $\diag(v)$ is the diagonal matrix with entries $v$. Using an MSE approach we end up with the following optimization problem
\begin{align}
    \min_{\Pm \in \R^N, \Dm \in \R^{N \times m}} \:\quad & \Vert \Wm \diag(\Pm) \Dm - \diag(\Wm \Pm)\Rm \Vert_F^2 \label{eq:Exp1:opt1} \\ 
    \textnormal{s.t.}\: \quad  & \Pm^\top\1 = 1 \nonumber\\
    &\Dm \1 = \1 \nonumber\\
    &\Pm,\Dm \geq 0 \nonumber
\end{align}
alternatively, we can also consider the more direct objective by only taking the norm of \cref{eq:Exp1:problem3} which also allows us to drop the first constraint.

It is easy to see that we can approximate $\Dm$ given $\Pm$ by solving a non-negative least square problem. Similarly, $\Pm$ can be found given $\Dm$ using a similar approach. Hence, we suggest a coordinate descent scheme, starting at $\Pm = \1/N$. The algorithm is presented in \cref{algo:opt}, where $\textsc{NNLL}$ is a non-negative least squares solver, $\textsc{Normalize}$ normalizes column sums to 1, and $(\Wm_{it}\Dm_{tj})_{(i,j),t} \in \R^{(n \cdot m)\times N}$ denotes the matrix obtained by flattening $i,j$ into one dimension.

\section{Experiments}
\label{sec:exp}
To showcase the more practical aspects of our consideration we conduct two numerical analyses. In the first, we evaluate \cref{algo:opt} in a reconstruction task. In the second, we consider a use case in the field of water distribution networks in which we are interested in statistical properties of $\D_t$ and $\PT$.

\subsection{Reconstructing the Distribution Process}
\label{sec:Exp1}

\begin{table}[t]
    \centering
    \caption{Result of reconstruction experiment (mean and std. over 500 runs of (median) negative logarithm of difference; $\approx$ correct decimals, larger is better). Columns: D / P reconstruction error of $\Dm$ / $\Pm$, f objective (according to \cref{eq:Exp1:opt1}).}
    {\footnotesize
    \begin{tabular}{l@{\:\:}r@{$\pm$}l@{\;}r@{$\pm$}l@{\;}r@{$\pm$}l@{\quad}r@{$\pm$}l@{\;}r@{$\pm$}l@{\;}r@{$\pm$}l@{\quad}r@{$\pm$}l@{\;}r@{$\pm$}l@{\;}r@{$\pm$}l@{\quad}r@{$\pm$}l@{\;}r@{$\pm$}l@{\;}r@{$\pm$}l}
\toprule
 & \multicolumn{6}{c}{Algo.~\ref{algo:opt} (ours)} & \multicolumn{6}{c}{Nelder-Meat \eqref{eq:Exp1:opt1}} & \multicolumn{6}{c}{SLSQP \eqref{eq:Exp1:opt1}} & \multicolumn{6}{c}{SLSQP \eqref{eq:Exp1:problem3}} \\
rank & \multicolumn{2}{c}{D} & \multicolumn{2}{c}{P} & \multicolumn{2}{c}{f} & \multicolumn{2}{c}{D} & \multicolumn{2}{c}{P} & \multicolumn{2}{c}{f} & \multicolumn{2}{c}{D} & \multicolumn{2}{c}{P} & \multicolumn{2}{c}{f} & \multicolumn{2}{c}{D} & \multicolumn{2}{c}{P} & \multicolumn{2}{c}{f} \\
\midrule
1.0 & 29 & 8 & 3 & 1 & 30 & 7 & 30 & 6 & 2 & 0 & 12 & 2 & 1 & 0 & 2 & 0 & 14 & 3 & 7 & 5 & 2 & 0 & 20 & 1 \\
1.4 & 30 & 7 & 4 & 1 & 28 & 6 & 30 & 7 & 2 & 0 & 12 & 2 & 2 & 1 & 2 & 0 & 14 & 1 & 8 & 6 & 3 & 0 & 19 & 2 \\
1.9 & 31 & 5 & 5 & 1 & 30 & 5 & 30 & 7 & 2 & 0 & 12 & 2 & 3 & 3 & 2 & 0 & 14 & 2 & 10 & 7 & 3 & 0 & 19 & 2 \\
2.2 & 31 & 5 & 6 & 1 & 31 & 5 & 29 & 8 & 2 & 0 & 11 & 2 & 4 & 4 & 3 & 0 & 15 & 2 & 12 & 7 & 3 & 0 & 19 & 2 \\
2.6 & 31 & 5 & 6 & 1 & 32 & 4 & 29 & 8 & 2 & 0 & 11 & 2 & 5 & 4 & 3 & 0 & 16 & 2 & 11 & 7 & 3 & 0 & 19 & 2 \\
\bottomrule
\end{tabular}}
    \label{tab:resultsExp1}
\end{table}

To empirically evaluate \cref{algo:opt} we considered a reconstruction task, i.e., we start with $\Dm,\Pm,\Wm$ from which we compute $\Rm$ and then reconstruct $\Dm$ and $\Pm$ form $\Wm$ and $\Rm$. In our experiment we considered $|\X| = m = 5$ and randomly constructed intervals as time windows, we refer to the number of windows per time-point $N/n$ as the ``rank'' of the problem. Besides \cref{algo:opt} we also considered SLSQP and Nelder-Mead as solvers, using \cite{LaueMG2019} to compute the gradients. In the latter two cases, we considered the problem as stated in \eqref{eq:Exp1:opt1} as well as using \cref{eq:Exp1:problem3} directly, whereby we drop the first constraint in the second case. In the case of Nelder-Mead, we only optimize $\Pm$ and then find $\Dm$ as in line~\ref{algo:opt:compteD}.

We repeated the procedure 500 times for different ranks. The results are shown in \cref{tab:resultsExp1}. Furthermore, we performed a rank analysis showing that our approach performed overall significantly best in all cases (post-hoc Nemenyi test at $\alpha = 0.001$). SLSQP on \cref{eq:Exp1:problem3} performed significantly better than on \eqref{eq:Exp1:opt1}, both Nelder-Mead approaches performed similarly in all cases. For the objective and optimizing $\Pm$ SLSQP outperformed Nelder-Mead, for $\Dm$ the opposite holds. This implies that line~\ref{algo:opt:compteD} constitutes an efficient way to estimate $\Dm$.
In particular, we found an excellent reconstruction of assuming a rank $\geq 2$. 

\subsection{Application to the Water-Domain}
\label{sec:Exp2}
\begin{figure}[t]
    \begin{subfigure}[t]{0.32\textwidth}
    \centering
    \includegraphics[width=\textwidth]{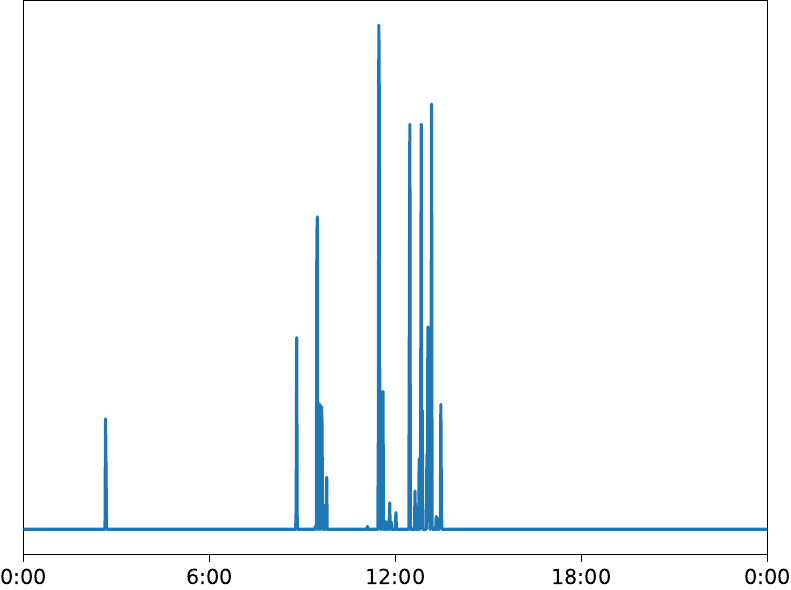}
    \subcaption{Consumption of one household over a day
    \label{fig:water-household}}
    \end{subfigure}
    \hfill
    \begin{subfigure}[t]{0.32\textwidth}
    \centering
    \includegraphics[width=\textwidth]{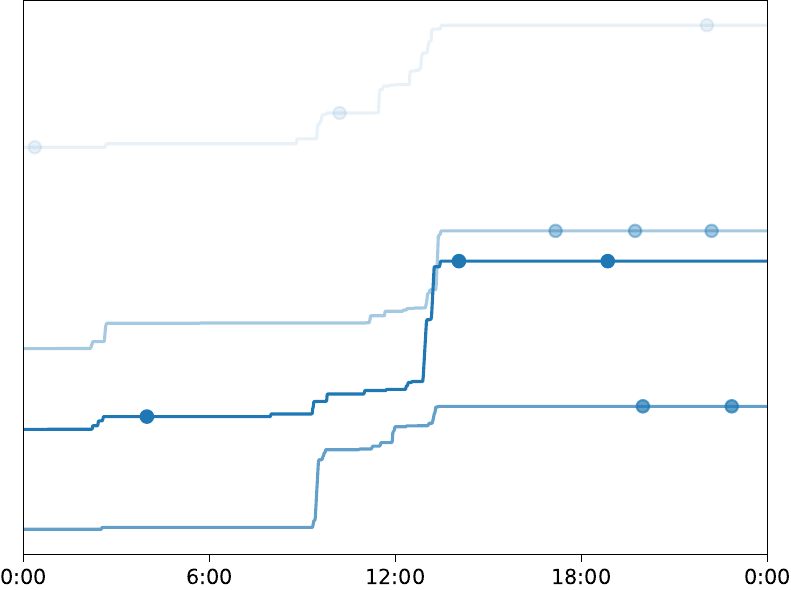}
    \subcaption{Cumulative consumption, observations marked\label{fig:water-meter-readings}}
    \end{subfigure}
    \hfill
    \begin{subfigure}[t]{0.32\textwidth}
    \centering
    \includegraphics[width=\textwidth]{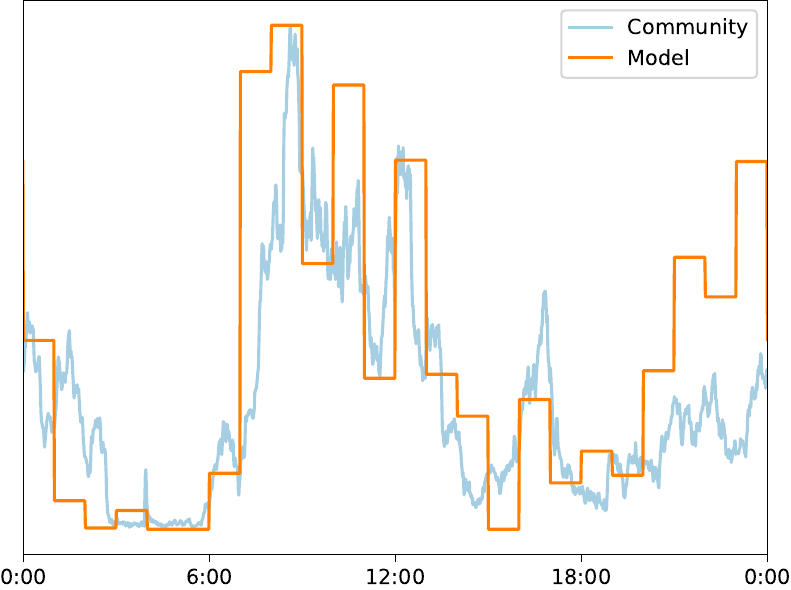}
    \subcaption{Community consumption and model prediction\label{fig:water-pred}}
    \end{subfigure}
    \caption{Experiment on water data. Figure shows original consumption data (a; unknown), cumulative consumption (b; unknown) and training data (marks in b; known), as well as time point-wise reconstruction and goal prediction (c).}
    \label{fig:water}
\end{figure}
Finally, we consider a more practical problem. Water distribution networks (WDNs) are a crucial part of the critical infrastructure supplying private and commercial customers reliably with clean and safe drinking water. To ensure the functionality of the system, planning and control based on water demand estimates are mandatory \cite{vaquet_challenges_2024}. However, WDNs are rarely equipped with smart meters measuring demands due to the associated costs and privacy concerns, making demand estimation a tough problem.
In this case study we assume that instead of having information on the actual water consumption at every time point as visualized in \cref{fig:water-household}, the accumulated water consumption is read out from traditional water meters at some scarce time points, i.e., we obtain the cumulative consumption of a given household (see \cref{fig:water-meter-readings}). The goal is to predict the typical water demand patterns accumulated over multiple households. Having this information, in particular the peak demands, available is important for the planning and control of the network.

Observe that this task relates to our previous considerations as the readings of the water meters correspond to the weighted window means $\PT(W)\D_W$, while the peak consumption corresponds to quantiles of $\D_t$. 
To approach the problem, we model the consumption of a single household using an inhomogeneous compound Poisson process, i.e., we assume that at random times a random amount of water is consumed. As a consequence, we can model communal consumption as a normal distribution whose parameters can be derived from the Poisson processes. This allows us to simplify \cref{algo:opt} to the estimation of $\Dm$ only as the moment of consumption has no influence on how much the water meter counts.

For our experiment we simulated $12{,}000$ households using \cite{steffelbauer2024pysimdeum} over 28 days, assuming 4 reports per day on average. We used $3{,}500$ households to fit the model to make hourly estimates of the consumption. The raw data as well as the derived model for one day are presented in \cref{fig:water-pred}. As can be seen, our model captures the overall consumption quite well. Especially it estimates the peak demands very precisely.

\section{Concluding Remarks and Future Work}
\label{sec:conclusion}

In this paper, we introduced a novel framework for handling data over time and concept drift by focusing on a time-window-based approach rather than the traditional point-in-time perspective. Our framework aligns more closely with practical, algorithmic needs, making it a useful tool for real-world applications where data distributions evolve. We compared our framework to the existing ones on a theoretical level and found that both model the same situation despite their different starting points, providing yet another justification for the 'time-as-a-feature' paradigm. We numerically evaluated our approach through experiments showing that our window-based setup can effectively reconstruct the information captured by the point-wise setup, thereby validating the equivalence numerically.
To illustrate the practical impact of our framework, we applied it to a case study in critical infrastructure, specifically within water distribution networks. This application highlights that our consideration can provide meaningful insights for considering dynamic environments, showing the relevance of our ideas beyond purely theoretical insights.

Unlike the framework proposed by \cite{dawidd}, which is restricted to finite horizons, we conjecture that our framework has the potential to handle infinite temporal horizons. Such considerations offer a foundation to statistically model infinite data streams which we will address in future work.

\bibliography{bib}

\begin{thebibliography}{1}

\bibitem{gama}
J.~Gama, I.~Žliobaitė, A.~Bifet, M.~Pechenizkiy, and A.~Bouchachia.
\newblock A survey on concept drift adaptation.
\newblock {\em ACM Comput. Surv.}, 46(4), March 2014.

\bibitem{oneortwo}
F.~Hinder, V.~Vaquet, and B.~Hammer.
\newblock One or two things we know about concept drift—a survey on
  monitoring in evolving environments. part a: detecting concept drift.
\newblock {\em Frontiers in Artificial Intelligence}, 7:1330257, 2024.

\bibitem{dawidd}
Fabian Hinder, Andr{\'e} Artelt, and Barbara Hammer.
\newblock Towards non-parametric drift detection via dynamic adapting window
  independence drift detection (dawidd).
\newblock In {\em ICML}, pages 4249--4259, 2020.

\bibitem{ida2023}
Fabian Hinder, Valerie Vaquet, Johannes Brinkrolf, and Barbara Hammer.
\newblock On the change of decision boundary and loss in learning with concept
  drift.
\newblock In {\em IDA}, pages 182--194, 2023.

\bibitem{LaueMG2019}
S\"{o}ren Laue, Matthias Mitterreiter, and Joachim Giesen.
\newblock {GENO} -- {GEN}eric {O}ptimization for classical machine learning.
\newblock In {\em NeurIPS}. 2019.

\bibitem{lu2018}
Jie Lu, Anjin Liu, Fan Dong, Feng Gu, Joao Gama, and Guangquan Zhang.
\newblock Learning under concept drift: A review.
\newblock {\em IEEE TKDE}, 31(12):2346--2363, 2018.

\bibitem{shalev2014understanding}
Shai Shalev-Shwartz and Shai Ben-David.
\newblock {\em Understanding machine learning: From theory to algorithms}.
\newblock Cambridge university press, 2014.

\bibitem{steffelbauer2024pysimdeum}
David Steffelbauer, Bram Hillebrand, and Mirjam Blokker.
\newblock pysimdeum-an open-source stochastic water demand end-use model.
\newblock In {\em WDSA \& CCWI}, 2024.

\bibitem{vaquet_challenges_2024}
Valerie Vaquet, Fabian Hinder, André Artelt, Inaam Ashraf, Janine Strotherm,
  Jonas Vaquet, Johannes Brinkrolf, and Barbara Hammer.
\newblock Challenges, {Methods}, {Data} -- a {Survey} of {Machine} {Learning}
  in {Water} {Distribution} {Networks}.
\newblock In {\em ICANN}. Springer, 2024.

\end{thebibliography}

\Arxiv{}{
\newpage
\section{Proofs}
In the following, we will provide proofs for the formal statements given in the paper.

\subsection{Proof of \cref{prop:induced_WDS}}
The proof mainly consists of a sequence of straightforward routine checks and simple computations. We only provide it for completeness:
\begin{proof}[\cref{prop:induced_WDS}]
It suffices to consider $i$ as $I = i \circ \holist$. 

Let $P$ be a probability measure on $\X \times \T$. Define $\PT(S) = P(\X \times S)$ the marginalization on $\T$ and $\W = \{S \in \Sigma_\T \mid \PT(S) > 0\}$, $\Null = \{S \in \Sigma_\T \mid \PT(S) = 0\}$ and $\D_W(A) = P(A \times W \mid \X \times W)$ for $W \in \W$.

\emph{$(\W,\Null)$ is a WS: }Obviously, $\Null \cap \W = \emptyset$.
For $S_1,S_2,\dots \in \Null$ we have 
\begin{align*}
    \PT\left(\bigcup_{n \in \N} S_n\right) \leq \sum_{n \in \N} P(S_n) = 0. 
\end{align*}
So $\bigcup_{n \in \N} S_n \in \Null$.
If $S_1 \subset S_2 \in \Null$ then $\PT(S_1) \leq \PT(S_2) = 0$ so $\PT(S_1) = 0$ and hence $S_1 \in \Null$. 
Similarly, if $S_1,S_2 \in \W$ then $0 < \PT(S_1) \leq \PT(S_1) + \PT(S_2 \setminus S_1) = \PT(S_1 \cup S_2)$ so $S_1 \cup S_2 \in \W$.
So $\Null$ is a $\sigma$-ideal, $\W$ is closed under finite unions, and $\Null$ and $\W$ are disjoint. 
Now, $\Null \cup \W = \{S \in \Sigma_\T \mid \PT(S) = 0 \text{ or } \PT(S) > 0\} = \Sigma_\T$ is a semi-ring and $\sigma(\Sigma_\T) = \Sigma_\T$ so it also generates $\Sigma_\T$ locally, and as $\T \in \W$ it also covers $\T$ and has a finite horizon. 
Hence, $(\W,\Null)$ are a WS.

\emph{$\D_W$ is a WDS: }
For $0 < a \leq b$ it holds $|1/a - 1/b| \leq |a-b| \frac{1}{a^2}$. Furthermore, for $W_1,W_2 \in \W$ with $\PT(W_1) \leq \PT(W_2)$ and $A \in \Sigma_\X$ and setting $\delta := \PT(W_1 \vartriangle W_2)$ it holds
\begin{align*}
    |P(A \times W_1) - P(A \times W_2)| 
    &\leq P((A \times W_1) \triangle (A \times W_2)) 
    \\&= P(A \times (W_1 \triangle W_2)) 
    \\&\leq \PT(W_1 \vartriangle W_2) = \delta
\end{align*}
Therefore,  we have
\begin{align*}
    &|\D_{W_1}(A)-\D_{W_2}(A)| 
    \\&=\quad \left|\frac{P(A \times W_1)}{\PT(W_1)} - \frac{P(A \times W_2)}{\PT(W_2)}\right|
    \\&\leq\quad \left|\frac{P(A \times W_1)}{\PT(W_1)} -\frac{P(A \times W_1)}{\PT(W_2)}\right|
    +\left|\frac{P(A \times W_1)}{\PT(W_2)}- \frac{P(A \times W_2)}{\PT(W_2)}\right|
    \\&\leq\quad \underbrace{\frac{P(A \times W_1)}{\PT(W_1)^2}}_{\leq \PT(W_1)^{-1}}\underbrace{\left|\PT(W_1) -\PT(W_1)\right|}_{=\delta}
    +\underbrace{\frac{1}{\PT(W_2)}}_{\leq \PT(W_1)^{-1}}\underbrace{\left|P(A \times W_1)- P(A \times W_2)\right|}_{\leq \delta}
    \\&\leq\quad \frac{2}{\PT(W_1)} \delta
\end{align*}
Hence, for $W_1 \vartriangle W_2 \in \Null, \; W_1,W_2 \in \W$ we have $\delta = 0$ and hence 
\begin{align*}
    |\D_{W_1}(A)-\D_{W_2}(A)| \leq \frac{2}{\PT(W_1)} \cdot 0 = 0
\end{align*}
and for $A \in \Sigma_\X$, $W_1 \subset W_2 \subset \dots \nearrow W, \; W,W_1,W_2,\dots \in \W$ and any $\varepsilon > 0$ choose $n$ such that $\PT(W) - \PT(W_n) \leq \frac{1}{4}\PT(W_1) \varepsilon$ then it holds
\begin{align*}
    |\D_{W_n}(A)-\D_{W}(A)| 
    &\leq \frac{2}{\PT(W_n)} \cdot \frac{\PT(W_1) \varepsilon}{4} 
    \\&= \frac{1}{2}\underbrace{\frac{\PT(W_1)}{\PT(W_n)}}_{\leq 1}\varepsilon < \varepsilon
\end{align*}

Now for $W_1,\dots,W_n \in \W$ pairwise disjoint we set $W = \cup_{i = 1}^n W_i$, $\lambda_i = \PT(W_i \mid W)$. Then for any subset $I \subset \{0,\dots,n\}$ and $W_I = \cup_{i \in I} W_i$ we have
\begin{align*}
    \D_{W_I}(A) 
    &= \frac{P(A \times W_I)}{\PT(W_I)}
    \\&= \frac{\sum_{i \in I} P(A \times W_i)}{\sum_{i \in I} \PT(W_i)}
    \\&= \frac{\sum_{i \in I} P(A \times W_i \mid \X \times W_i) \PT(W_i \mid W_I) \PT(W_I)}{\sum_{i \in I} \PT(W_i\mid W_I) \PT(W_I)}
    \\&= \frac{\sum_{i \in I} \lambda_i\D_{W_I}(A)}{\sum_{i \in I} \lambda_i}.
\end{align*}

To show that $i(P)$ is not constant if and only if $P$ has drift, recall that $\D$ has no drift if and only if $\D(A \times W) = \D_\T(A)\PT(W)$ for all $A,W$. Therefore, if $P$ has no drift then $\D_W(A) = \frac{P(A \times W)}{\PT(W)} = \frac{\D_\T(A)\PT(W)}{\PT(W)} = \D_\T(A)$ is constant. Conversely, to check that $P = \D_\T \times \PT$ it suffices to do this on an intersection stable generator like $\Sigma_\X \times \Sigma_\T$ on which it holds
\begin{align*}
    P(A \times W) 
    &= P(A \times W \mid \X \times W) \PT(W)
    \\&= \D_W(A) \PT(W)
    \\&= \D_\T(A) \PT(W)
    \\&= (\D_\T \times \PT)(A \times W).
\end{align*}
\end{proof}

\subsection{Proof of \cref{prop:premain}}
To prove \cref{prop:premain} we need a version of Fubini's theorem for elementary conditional probabilities which we will prove first. Once that is done, the actual proof is rather straightforward. 
Recall that Fubini's theorem for regular conditional probabilities states the following: For measurable spaces $(X,\Sigma_X)$ and $(Y,\Sigma_Y)$, a probability measure $P$ on $X$, and a Markov kernel $Q$ from $X$ to $Y$ there exists a unique probability measure $H$ on $(X \times Y, \Sigma_X \otimes \Sigma_Y)$ such that
\begin{align*}
    \int f(x,y) \d H(\,(x,y)\,) = \iint f(x,y) \d Q(y \mid x) \d P(x)
\end{align*}
for all bounded measurable $f : X \times Y \to \R$. We want to extend this theorem in the sense that if we take the conditional probabilities $Q_A(B) = H(A \times B \mid A \times Y) := H(A \times B) / H(A \times Y)$ and the marginal $P(A) = H(A \times Y)$ then we can reconstruct $H$. More formally:
\begin{theorem}[Fubini's theorem for elementary conditional probabilities]
    \label{lem:elementary_fubini}
	\newcommand{\F}{\mathcal{F}}
	Let $(X,\Sigma_X), (Y,\Sigma_Y)$ be measurable spaces with countably generated $\sigma$-algebras. Let $P$ be a probability measure on $X$, $\F \subset \Sigma_X$ be a subset of $P$ non-null sets, and $(Q_A)_{A \in \F}$ a family of probability measures on $Y$ such that
	\begin{enumerate}
		\item $\F \cup \Null$ in an algebra that contains a countable generator of $\Sigma_X$, where $\Null \subset \Sigma_X$ are the $P$ null sets. 
		\item For $A = A_1 \cup \dots \cup A_n$ disjoint, $A,A_1,\dots,A_n \in \F$ we have
		$$Q_A = \sum_{i = 1}^n P(A_i \mid A) Q_{A_i}$$
		\item For $A_1, A_2 \in \F$ with $P(A_1 \triangle A_2) = 0$ we have $Q_{A_1} = Q_{A_2}$
		\item For $A_1 \subset A_2 \subset \dots \nearrow A$ with $A,A_1,\dots \in \F$ we have $$\lim_{n \to \infty} Q_{A_n}(B) = Q_{A}(B)$$ for all $B \in \Sigma_Y$ 
	\end{enumerate}
	Then there exist a unique probability measure $H$ on $(X \times Y, \Sigma_X \otimes \Sigma_Y)$ such that $H(A \times Y) = P(A)$ for all $A \in \Sigma_X$ and $Q_A(B) = H(A \times B \mid A \times Y)$ for all $A \in \F, B \in \Sigma_Y$. 
\end{theorem}
\begin{proof}
	\newcommand{\Gx}{\mathcal{U}}
	\newcommand{\Gy}{\mathcal{V}}
	\newcommand{\Gi}{G_\infty}
	\newcommand{\F}{\mathcal{F}}
	\newcommand{\U}{\mathbf{U}}
	\newcommand{\V}{\mathbf{V}}
	\newcommand{\Ring}{\mathcal{R}}
	
	Let $\{U_1,U_2,\dots \} \subset \F \cup \Null$ a generator of $\Sigma_X$ and $\{V_1,V_2,\dots\} \subset \Sigma_Y$ a generator of $\Sigma_Y$. Denote by $\Gx_n = \sigma(\{U_1,\dots,U_n\}), \Gx = \cup_{n=1}^\infty \Gx_n$ and $\Gy_n = \sigma(\{V_1,\dots,V_n\}), \Gy = \cup_{n=1}^\infty \Gy_n$.
	
	We will make use of the following notation: For any set $A$ we write $A^i = A \times \dots \times A$ for taking the product of $A$ with itself $i$-times, we write $A^0 \times B = B$, and $A^{-1} = A^C$ for the complement of $A$. We denote by $Z = X \times Y$. And for any $m$ and any sequence $u \in \{-1,1\}^n$ of length $n$ with $m \leq n$ we write $\U_m^u = \cap_{i = 1}^m U_i^{u_i}$ and analog for $V_i$.  
	We write $\Gx_n * \Gy_n = \{A \times B \mid A \in \Gx_n, B \in \Gy_n\}$ and analog for $\Gx$, $\Gy$, etc. and $\Sigma_X \otimes \Sigma_Y = \sigma(\Sigma_X * \Sigma_Y)$.

	As $\Gx*\Gy$ contains an intersection stable generator on which the values of $H$ are defined, there can exist at most one measure that fulfills the requirements on $H$.
	
	To construct $H$ approximate it successively on $\Gx_n * \Gy_n$ as follows:
	
	Set $K_1$ as the measure which is given by the (weighted) sum of product measures $P(\cdot \mid U_1)Q_{U_1}$ and $P(\cdot \mid U_1^C)Q_{U_1^C}$ weighted by $P(U_1)$ and $P(U_1^C) = 1-P(U_1)$, respectively, i.e.,
	\begin{align*}
		K_0(A \times B) &= P(A \cap U_1)Q_{U_1}(B) + P(A \cap U_1^C)Q_{U_1^C}(B)
	\end{align*}
	
	Analogous define $K_i$ as the weighted sum of product measures as follows
	\begin{align*}
		K_i(A \times B \mid (x,y) \; ) &= \sum_{u,v \in \{-1,1\}} P(A \cap U_i^u \mid \U_{i-1}^{s(x)}) \\&\qquad\qquad\qquad\:\cdot\:\frac{Q_{U_i^u \cap \U_{i-1}^{s(x)}}(\V_{i-1}^{s(y)})}{Q_{\U_{i-1}^{s(x)}}(\V_{i-1}^{s(y)})} \\&\qquad\qquad\qquad\:\cdot\:Q_{U_i \cap \U_{i-1}^{s(x)}}(B \cap V_i^v \mid \V_{i-1}^{s(y)}) 
	\end{align*} 
	where $s(x) \in \{-1,1\}^\N$ with $s(x)_i = 1$ if $x \in U_i$ and $-1$ otherwise and analogous for $s(y)$. Here we ignore the summands with $P(U^u_n) = 0$ so that $\U^u_n \in \F$ for all remaining $u$. If $P(\U_{i-1}^{s(x)}) = 0$ we set $K_i(\cdot\mid (x,y)) = \delta_{(x,y)}$.
	
	By the Ionescu-Tulcea theorem there exists a measure $G$ on $((X \times Y)^\infty, \sigma(\otimes^\infty (\Sigma_X \otimes \Sigma_Y)))$ 
	that extends the finite products. 
	
	We claim that the desired measure is $H$ is the limit $K_1 K_2 \dots$ which is a marginal distribution of $G$. 
	Observe that for all $n$ we have
	\begin{align*}
		G(Z^{n-1} \times (A \times B) \times Z^\infty) &= \sum_{u,v \in \{-1,1\}^{n}} P(A \cap \U_n^{u}) Q_{\U_n^{u}}(B \cap \V_n^{v})
	\end{align*}
	which is clear for $n = 1$ and follows by induction for $n > 1$ as 
	\begin{align*}
		&\sum_{u,v \in \{-1,1\}^{n}} P(X \cap \U_{n-1}^{u}) Q_{\U_{n-1}^{u}}(Y \cap \V_{n-1}^{v}) \\&\qquad\qquad\cdot P(A \cap U_n^{u_n} \mid \U_{n-1}^u) \frac{ Q_{\U_n^u}(\V_{n-1}^v)}{Q_{\U_{n-1}^u}(\V_{n-1}^v)}Q_{\U_n^u}(B \cap V_n^{v_n} \mid \V_{n-1}^{v}) 
		\\&= \sum_{u,v \in \{-1,1\}^{n}} P(A \cap \underbrace{U_n^{u_n} \cap \U_{n-1}^u}_{= \U_n^u}) Q_{\U_n^u}(B \cap \underbrace{V_n^{v_n} \cap \V_{n-1}^{v}}_{=\V_n^v}).
	\end{align*}
	Furthermore, for all $n$ and $A \times B \in \Gx_n * \Gy_n$ we have
	\begin{align*}
	    G(Z^{n-1} \times (A \times B) \times Z^\infty) 
	    &= \sum_{u,v \in \{-1,1\}^{n}} P(A \cap \U_n^{u}) Q_{\U_n^{u}}(B \cap \V_n^{v}) 
	    \\&= P(A) \sum_{u \in \{-1,1\}^{n}} \underbrace{P(\U_n^{u} \mid A)}_{=\1[\U_n^u \subset A]} \underbrace{\sum_{v \in \{-1,1\}^{n}}Q_{\U_n^{u}}(B \cap \V_n^{v})}_{=Q_{\U_n^u}(B)}
	    \\&= P(A)Q_A(B)
	\end{align*}
	where the last equality is condition 2. Similarly
	\begin{align*}
		K_{n+1}(A \times B \mid (x,y) \; ) = \1[x \in A, y \in B]
	\end{align*}
	from which we can conclude that for all $m \geq 1$ we have
	\begin{align}
		&G(Z^{n-1} \times (A \times B)^m \times (A \times B)^C \times Z^\infty ) \label{eq:appendix:AxAceq0}
		\\&\qquad\leq\quad G(Z^{n-1+m-1} \times (A \times B) \times (A \times B)^C \times Z^\infty) = 0. \nonumber
	\end{align}
	And by adding this in, we get
	\begin{align*}
		&G(Z^{n-1} \times (A \times B) \times Z^\infty) 
		\\&\qquad=\quad G(Z^{n-1} \times (A \times B) \times (A \times B) \times Z^\infty)
		\\&\qquad\quad+ \underbrace{G(Z^{n-1} \times (A \times B) \times (A \times B)^C \times Z^\infty)}_{= 0}
		\\&\qquad=\quad G(Z^{n-1} \times (A \times B)^2 \times Z^\infty)
		\\&\qquad\:\:\vdots
		\\&\qquad=\quad G(Z^{n-1} \times (A \times B)^m \times Z^\infty)
		\\&\qquad=\quad G(Z^{n-1} \times (A \times B)^\infty)
		\\&\qquad=\quad G(Z^{n-1} \times (A \times B) \times (A \times B)^\infty) \\&\qquad\quad+ \underbrace{G(Z^{n-1} \times (A \times B)^C \times (A \times B)^\infty)}_{= 0}
		\\&\qquad=\quad G(Z^{n} \times (A \times B)^\infty) 
		\\&\qquad\:\:\vdots
		\\&\qquad=\quad G(Z^{n-1+m-1} \times (A \times B)^\infty)
	\end{align*}
	where the forth equality follows from the fact that $G$ is a measure and the sequence of sets is monotonously decreasing. So in particular, we have
	\begin{align*}
		G(Z^{n-1+m-1} \times (A \times B)^\infty) = P(A)Q_A(B).
	\end{align*}
	
	Therefore, the marginal distribution of $G$ at infinity is indeed the measure $H$. We now need to extract the measure from the sequence. To do so, we will show that $\lim_{n \to \infty} G(Z^n \times C^\infty)$ is a pre-measure and then extend it. 
	
	Define $G_n(C) = G(Z^{n-1} \times C^\infty), \; G_\infty(C) = \lim_{n \to \infty}G_n(C)$ which exists as the limit of an increasing bounded sequence and for $A \in \Gx_n, B\in \Gy_n$ we have $\Gi(A \times B) = G_n(A \times B) = P(A)Q_A(B)$. We have to show that $\Gi$ is a measure.
	
	Denote by $\Ring_n$ and $\Ring$ the ring over $\Gx_n*\Gy_n$ or $\Gx * \Gy$, respectively. Every set $C \in \Ring$ is a finite disjoint union $C = C_1 \sqcup \dots \sqcup C_n$ with $C_i \in \Gx * \Gy$. Because $\Gx_n$ and $\Gy_n$ are monotonous, for every $A \times B \in \Gx * \Gy = \{A \times B \mid A \in \cup_{n=1}^\infty \Gx_n, B \in \cup_{n=1}^\infty \Gy_n\}$ there exists an $n$ such that $A \in \Gx_n$ and $B \in \Gy_n$ and therefore $A \times B \in \Gx_n * \Gy_n$.
	Hence, there is an $n$ such that $C_i \in \Gx_n*\Gy_n$ for all $i$ and thus $C \in \Ring_n$.
	
	For any disjoint sets $C_1,\dots,C_k$ set $C = C_1 \cup \dots \cup C_k$ then we have
	\begin{align*}
		(C_1 \cup \dots \cup C_k)^\infty
		&= \bigcup_{j = 1}^k \bigcup_{i=0}^\infty C_j^{(i)} & \text{where} \\
		C_j^{(i)} &= \begin{cases}
			C_j^\infty & \text{ if } i = 0\\
			C_j^i \times (C \setminus C_j) \times C^\infty & \text{otherwise}
		\end{cases}
	\end{align*}
	As $C_j^{(i)} \subset C_j^i \times (C \setminus C_j) \times Z^\infty$ for all $C_1,\dots,C_k \in \Gx_n * \Gy_n$ we have 
	\begin{align*}
		G(Z^{n-1} \times C_j^{(i)}) \leq G(Z^{n-1} \times C_j^i \times (C \setminus C_j) \times Z^\infty) = 0
	\end{align*}
	by \cref{eq:appendix:AxAceq0}. As we can write all $C_i \in \Ring_n$ as $C_{i1} \cup \dots \cup C_{it},\; C_{ij} \in \Gx_n*\Gy_n$ disjoint the statement extends to all $C_i \in \Ring_n$. Therefore, we have
	\begin{align*}
		G_{n}\left( \bigcup_{j = 1}^k C_j \right) &= 
		G(Z^{n-1} \times (C_1 \cup \dots \cup C_k)^\infty) 
		\\&= \sum_{j = 1}^k \underbrace{G(Z^{n-1} \times C_j^{(0)})}_{= G_n(C_j)} +  \sum_{j = 1}^k\sum_{i = 1}^\infty \underbrace{G(Z^{n-1} \times C_j^{(i)})}_{= 0} 
		\\&= \sum_{j = 1}^k G_n(C_j)
	\end{align*}

	Thus, $G_n$ and hence also $G_\infty$ are finitely additive. 
	
	Thus $\sigma$-additivity is equivalent to show continuity from below: for $C_1 \subset C_2 \subset \dots \nearrow C, \; C,C_1,C_2,\dots \in \Ring$ consider
	\begin{align*}
		\lim_{m \to \infty} \Gi(C_m) 
		&= \lim_{m \to \infty} \lim_{n\to \infty} G_n(C_m)
		\\&\overset{!^1}{=} \lim_{n\to \infty}\lim_{m \to \infty} G_n(C_m)
		\\&\overset{!^2}{=} \lim_{n\to \infty} G_n(C)
		\\&\overset{!^3}{=} \Gi(C),
	\end{align*}
	where 
	$!^1$ holds because $Z^n \times C_m^\infty \subset Z^{n+1} \times C_m^\infty$ so the double sequence $G_n(C_m)$ is increasing in both arguments $m$ and $n$ and therefore the limits commute,
	$!^2$ holds because $Z^n \times C_1^\infty \subset Z^n \times C_2^\infty \subset \dots \nearrow Z^n \times C^\infty$ for each $n$,
	and
	$!^3$ holds since as above $C$ is contained in some $\Ring_n$.
	
	Hence, as $G_n(\emptyset) = 0$ for all $n$ we see that $G_\infty$ is a pre-measure, indeed.
	By Carathéodory's extension theorem, we have that $\Gi$ extends uniquely to a measure on $\sigma(\Gx * \Gy) = \Sigma_X \otimes \Sigma_Y$. 
\end{proof}

Furthermore, we have the following quite helpful lemma:
\begin{lemma}[Carathéodory's extension theorem for WS with finite horizon]
\label{lem:extension}
Let $(\W,\Null)$ a WS on $(\T,\Sigma_\T)$ with finite horizon (and $(\X,\Sigma_\X)$ a measurable space). Then $\sigma(\W\cup\Null) = \Sigma_\T$. 

In particular, for every finite pre-measure $\mu$ on $\W \cup \Null$ ($\Sigma_\X * (\W \cup \Null)$) the outer measure $\mu^*$ is a measure on $\Sigma_\T$ ($\Sigma_\X \otimes \Sigma_\T$). This extension is unique, i.e., for every measure $\nu$ on $\Sigma_\T$ ($\Sigma_\X \otimes \Sigma_\T$) with $\nu_{|\W \cup \Null} = \mu$ ($\nu_{|\Sigma_\X * (\W \cup \Null)} = \nu$) it holds $\mu^*(S) = \nu(S)$ for all $S \in \Sigma_\T$ ($S \in \Sigma_\X \otimes \Sigma_\T$). 

Furthermore, if $\mu(N) = 0$ ($\mu(\X \times N) = 0$) for all $N \in \Null$ then the total value is determined on the finite horizon, i.e., for $W \in \W$ with $\T \setminus W \in \Null$ we have $\mu(W) = \mu^*(\T)$ ($\mu(\X \times W) = \mu^*(\X \times \T)$).
\end{lemma}
\begin{proof}
If $(\W,\Null)$ has a finite horizon, then there is a $W \in \W$ with $\T \setminus W \in \Null$. Thus, for $S \in \Sigma_\T$ we have $S = (S \cap W) \cup (S \setminus W)$ and $S \cap W \in \sigma(\W \cup \Null)$ and $S \setminus W \subset \T \setminus W \in \Null$ so $S \setminus W \in \Null$. Hence $S \in \sigma(\W \cup \Null)$.

The extension statement is just Carathéodory's extension theorem.

For the last part take $\mu(W) = \mu^*(W) + \mu^*(\T \setminus W) = \mu^*(\T)$ as $\T \setminus W \in \Null$ and thus $\mu^*(\T \setminus W) = 0$ by assumption. The same holds true when taking the product with $\X$.
\end{proof}
Clearly, this lemma can be extended to WS with non-finite horizons. However, as already discussed in the paper, this is beyond the scope of this paper and will not be covered here.

Using this we can now prove the following lemma that directly implies the proposition:
\begin{lemma}
\label{lem:reconstruct_HD}
Let $(\W,\Null,\D_W)$ be a WDS on finite horizons and $\PT$ a compatible time distribution, then there exists a unique distribution $\D \in \Pr(\X \times \T)$ such that $i(\D)$ is an extension of $\D_W$ and has $\PT$ as time marginal, i.e., $\PT(W) = \D(\X \times W)$. 
\end{lemma}
\begin{proof}
\newcommand{\Ring}{\mathcal{R}}
Let $\Ring = \Sigma_\X * (\W \cup \Null)$. Obviously, $\Ring \subset \Sigma_\X \otimes \Sigma_\T$ and as $\W \cup \Null$ is a semi-ring, so is $\Ring$. 
For $A \times S \in \Ring$ define
\begin{align*}
	\mu(A \times S) &= \begin{cases}
		\PT(S)\D_S(A) & \text{ if } S \in \W \\
		0 & \text{ otherwise } 
	\end{cases}.
\end{align*}
We claim that $\mu$ is a pre-measure on $\Ring$.

Obviously, $0 \leq \mu(\emptyset) = \mu(\X \times \emptyset) = 0$ as $\emptyset \in \Null$. For $A_1 \times W_1,A_2 \times W_2,\dots \in \Ring$ pairwise disjoint with $A \times W = \cup_{i = 1}^\infty A_i \times W_i \in \Ring$ observe that we can apply \cref{lem:elementary_fubini} to $X = (W,\sigma(\{W_1,W_2,\dots\}))$, $Y = (\X, \sigma(\{A_1,A_2,\dots\}))$, $\mathcal{F} = \{W' \cup N' \mid W' \in \W, N' \in \Null\} \cap \sigma(\{W_1,W_2,\dots\})$ (as a $\sigma$-algebra on $W$), $P = \PT$, and with $Q_{W\cup N} = \D_W,\; W\in\W,N\in\Null$ where the latter is well-defined as $Q_W$ and $\D_W$ are invariant under changes on null windows. Thus, there is a measure $H$ such that $\mu(A' \times W') = H(A' \times W')$ for all $A' \in \sigma(\{A_1,A_2,\dots\})$ and $W' \in \mathcal{F}$, so in particular, $W' \in \{W,W_1,W_2,\dots\}$. Therefore, we have
\begin{align*}
	\mu\left( \bigcup_{i = 1}^\infty A_i \times W_i \right) 
	&= H\left( \bigcup_{i = 1}^\infty A_i \times W_i \right)
	\\&= \sum_{i = 1}^\infty H\left(  A_i \times W_i \right)
	\\&= \sum_{i = 1}^\infty \mu\left( A_i \times W_i \right),
\end{align*}
so $\mu$ is $\sigma$-additive.
As $(\W,\Null)$ has a finite horizon and $\mu(A \times N) = 0$ for all $N \in \Null$ by \cref{lem:extension}, the outer measure $\mu^*$ is the unique measure that extends $\mu$ and is a probability measure.
Therefore, $\D = \mu^*$ is the only probability measure with $\D(\X \times W) = \PT(W)$ and $\D(A \times W \mid \X \times W) = \D_W(A)$. 

Now let $(\bar\W,\bar\Null,\bar\D) = i(\D)$. For $S \in \W$ we have $\PT(S) > 0$ and conversely $S \in \bar\W$ if and only if $0 < \D(\X \times S) = \PT(S)$ so $\W \subset \bar\W$. Analogous, we see $\Null \subset \bar\Null$. Furthermore, $\bar\D_W(A) = \D(A \times W \mid \X \times W) = \D_W(A)$ so $i(\D)$ is an extension of $(\W,\Null,\D_W)$.
\end{proof}

\begin{proof}[\cref{prop:premain}]
Observe that for $\D \in \Pr(\X \times \T)$ we have that $i(\D)$ and $\PT(W) = \D(\X \times W)$ are compatible. Thus, by \cref{lem:reconstruct_HD}, there is a unique distribution $\D' \in \Pr(\X \times \T)$ such that $i(\D) = i(\D')$ and therefore $\D = \D'$. The statement about the right concatenation is also already shown in the lemma.
\end{proof}

\subsection{Proof of \cref{prop:exist_PT}}
The proof of \cref{prop:exist_PT} is somewhat lengthy. For clarity we subdivided into a sequence of lemmas:

\begin{lemma}
\label{lem:prop_exist:help1}
Let $(\W,\Null,\D_W)$ be a non-constant WDS with $|\W| < \infty$. Fixate $W_* \in \W$. Then there exists a unique measure $\PT$ on $\T$ that is compatible with $\D_W$ and $\PT(W_*) = 1$.

If $(\tilde\W,\tilde\Null,\tilde\D_W)$ is another non-constant WDS on a measure space $(\T,\mathcal{S})$ with fewer measurable sets, i.e., $\mathcal{S} \subset \Sigma_\T$, with $W_* \in \tilde\W\subset\W$, $\tilde\Null \subset \Null$, $\tilde\D_W = \D_W$ for all $W \in \tilde\W$, then the obtained $\tilde\PT$ is the restriction of $\PT$ to $\mathcal{S}$, i.e., $\tilde\PT = (\PT)_{|\mathcal{S}}$.
\end{lemma}
\begin{proof}
\newcommand{\Ring}{\mathcal{R}}
Observe that there exists a finite set of pairwise disjoint $W_1,\dots,W_n \in \W$ such that $\Null \cup \{W_1,\dots,W_n\}$ is a semi-ring $\Ring$ with $\sigma(\Ring) = \sigma(\Null \cup \W)$. Since $W \setminus (W_1 \cup \dots \cup W_n) \in \Null$ for all $W \in \W$ we have that $(\W,\Null)$ has a finite horizon and thus $\sigma(\Ring) = \Sigma_\T$ by \cref{lem:extension}.

By definition there are $\lambda_1,\dots,\lambda_n$ such that for every $I \subset \{1,\dots,n\}$ we have
\begin{align*}
	\D_{W_I} = \frac{\sum_{i \in I} \lambda_i \D_{W_i}}{\sum_{i \in I} \lambda_i}
\end{align*}
where $W_I = \cup_{i \in I} W_i$.

As $\D_W$ is not constant w.l.o.g. w.m.a. $\D_{W_1} \neq \D_{W_2}$. 
Thus, there exists a unique $\lambda$ such that $\D_{W_1 \cup W_2} = \lambda\D_{W_1} + (1-\lambda)\D_{W_2}$ and as $\lambda_2/(\lambda_1+\lambda_2) = 1-\lambda$ every choice of $\lambda_1$ uniquely determines $\lambda_2$.

Furthermore, for every $i$ we have $\D_{W_i} \neq \D_{W_1}$ or $\D_{W_i} \neq \D_{W_2}$ then by the same argument as before we see that $\lambda_i / (\lambda_1+\lambda_i)$ or $\lambda_i / (\lambda_2+\lambda_i)$ is uniquely determined and therefore $\lambda_i$ is uniquely determined by $\lambda_1$. Hence, $\lambda_1$ determines the value of $\lambda_2,\dots,\lambda_n$ uniquely.

Now, set 
$\PT(S) = \sum_{i = 1}^n \lambda_i \1[W_i \setminus S \in \Null]$. 
Obviously, $\PT(S) = \PT(S\cup N)$ for all $N \in \Null$ and therefore $\PT(S_1) = \PT(S_2)$ if $S_1 \triangle S_2 \in \Null$. Hence, as every $W \in \W$ is up to a set in $\Null$ a disjoint union of $W_1,\dots,W_n$ we see that $\PT$ is finitely additive, and as $n$ is finite it is thus a pre-measure.
Therefore, by \cref{lem:extension}, it extends uniquely to a measure on $\Sigma_\T$. 
Normalizing $\PT$ to assure $\PT(W_*) = 1$ completes the construction.

For the the extension property use that we may choose the same $\lambda_i$ as before. As those are uniquely determined up to scaling the statement follows.
\end{proof}

\begin{lemma}
\label{lem:prop_exist:help2}
Let $(\W,\Null,\D_W)$ be a non-constant WDS with finite horizon. Any finite content $\PT$ on $\W \cup \Null$ that is compatible with $\D_W$, i.e., fulfills conditions 1.-3. of \cref{def:comp_PT}, extends uniquely to a compatible measure.
\end{lemma}
\begin{proof}
By \cref{lem:extension}, it suffices to show that $\PT$ is $\sigma$-additive. 
Let $W_1,W_2,\dots \in \W \cup \Null$ with $\cup_{i} W_i \in \W \cup \Null$. Set $N = \{i \mid W_i \in \Null\}$ then it holds
\begin{align*}
    \PT\left(\bigcup_{i = 1}^\infty W_i \right) 
    &= \PT\left(\left(\bigcup_{i \in N} W_i\right) \cup \left(\bigcup_{i \not\in N} W_i\right)\right) 
    \\&= \PT\left(\bigcup_{i \in N} W_i\right) + \PT\left(\bigcup_{i \not\in N} W_i\right)
    \\&= \PT\left(\bigcup_{i \not\in N} W_i\right)
\end{align*}
where the last equality follows because $\Null$ is a $\sigma$-ideal and thus $\bigcup_{i \in N} W_i \in \Null$. Furthermore, because $\bigcup_{i \not\in N} W_i = W \setminus \bigcup_{i \in N} W_i$ we have $\bigcup_{i \not\in N} W_i \in \W$. 

Thus, w.l.o.g. w.m.a. $W_1,W_2,\dots \in \W$. Alter the sequence at finitely many places to obtain a second sequence $V_1,V_2,\dots \in \W$ with 
\begin{align*}
    \D_{\cup_{i \in \N} W_i}(A) \neq \D_{\cup_{i \in \N} V_i}(A)
\end{align*}
for some $A$.

To see that this can be done proceed as follows: 
As $\D_W$ is not constant there are $U_1,U_2 \in \W$ with $\D_{U_1} \neq \D_{U_2}$. As $U_1\setminus U_2, U_2 \setminus U_1, U_1 \cap U_2, \T \setminus (U_1 \cup U_2)$ is a finite cover of $\T$ we may assume that every $W_i$ is contained in one of them as $\PT$ is additive. Now, extend (if necessary) $W_i$ to $\tilde{W}_i$ so that we can write $U_1$ and $U_2$ as disjoint unions of sets in $\tilde{W}_i$. Hence, as $W \mapsto \D_W$ is continuous, there are $j,k$ such that $\D_{\tilde W_j} \neq \D_{\tilde W_k}$ and w.l.o.g. $\D_{\tilde W_j} \neq \D_{\cup_{i \in \N}\tilde W_i}$ and therefore $\D_{\tilde W_j} \neq \D_{\cup_{i \neq j} \tilde W_i}$ as the coefficients $\lambda_i$ in \cref{def:WDS} are strictly larger 0. Hence, $\D_{\cup_{i \in \N} W_i}$ must either differ from $\D_{\cup_{i \in \N} \tilde W_i}$ or $\D_{\cup_{i \neq j} \tilde W_i}$ and both differ from $W_i$ in finitely many places only. Define $V_i$ accordingly. 

As $V_i$ and $W_i$ differ in finitely many places only we have
\begin{align*}
    \PT\left(\bigcup_{i = 1}^\infty V_i\right) = \sum_{i = 1}^\infty \PT(V_i) \qquad\Leftrightarrow&\qquad \PT\left(\bigcup_{i = 1}^\infty W_i\right) = \sum_{i = 1}^\infty \PT(W_i).
\end{align*}
Therefore assume for the sake of contradiction that neither
\begin{align*}
    \lambda_k := \frac{\PT\left(\bigcup_{i = 1}^k W_i\right)}{\PT\left(\bigcup_{i = 1}^\infty W_i\right)} \qquad\qquad \mu_k := \frac{\PT\left(\bigcup_{i = 1}^k V_i\right)}{\PT\left(\bigcup_{i = 1}^\infty V_i\right)}
\end{align*}
converge to 1, i.e., as $0 \leq \lambda_k \leq \lambda_{k+1} \leq 1$ the limit $\lambda_\infty = \lim_{k \to \infty} \lambda_k$ exists and we assume $\lambda_\infty < 1$ and analogous for $\mu_k,\mu_\infty$. It holds
\begin{align}
    \D_{\cup_{i = 1}^\infty W_i}(A) &= \lambda_k \D_{\cup_{i = 1}^k W_i}(A) + (1-\lambda_k) \D_{\cup_{i = k+1}^\infty W_i}(A). \label{eq:appendix:lemma_eq}
\end{align}
By condition 3. of \cref{def:WDS} $\lim_{k \to \infty} \D_{\cup_{i = 1}^k W_i}(A) = \D_{\cup_{i = 1}^\infty W_i}(A)$ and thus
\begin{align*}
    a_k := \frac{\D_{\cup_{i = 1}^\infty W_i}(A) - \lambda_k \D_{\cup_{i = 1}^k W_i}(A)}{1-\lambda_k} = \D_{\cup_{i = k+1}^\infty W_i}(A)
\end{align*}
is well-defined and the limit $k\to\infty$ exists. Rewriting \cref{eq:appendix:lemma_eq} using $a_k$ and taking limits we see that
\begin{align*}
    \lim_{k \to \infty}\D_{\cup_{i = 1}^\infty W_i}(A) 
    &= \lim_{k \to \infty}\left(\lambda_k \D_{\cup_{i = 1}^k W_i}(A) + (1-\lambda_k) a_k\right)
    \\&= \lambda_\infty \D_{\cup_{i = 1}^\infty W_i}(A) + (1-\lambda_\infty) \lim_{k \to \infty} a_k
    \\\Rightarrow (1-\lambda_\infty)\D_{\cup_{i = 1}^\infty W_i}(A) &= (1-\lambda_\infty)\lim_{k \to \infty} a_k
    \\\Rightarrow \D_{\cup_{i = 1}^\infty W_i}(A) &= \lim_{k \to \infty} a_k
\end{align*}

Performing the same argument based on $V_i$ we see
\begin{align*}
    \D_{\cup_{i = 1}^\infty W_i}(A) 
    &= \lim_{k \to \infty} \D_{\cup_{i = k+1}^\infty W_i}(A)
    \\&\overset{!}{=} \lim_{k \to \infty} \D_{\cup_{i = k+1}^\infty V_i}(A)
    \\&= \D_{\cup_{i = 1}^\infty V_i}(A) 
\end{align*}
where $!$ holds because $W_i$ and $V_i$ differ at finitely many positions only so that for $k \gg 0$
\begin{align*}
         \D_{\cup_{i = k+1}^\infty W_i}(A) 
      &= \lim_{n \to \infty}\D_{\cup_{i = k+1}^n W_i}(A) 
    \\&= \lim_{n \to \infty}\D_{\cup_{i = k+1}^n V_i}(A)
    \\&= \D_{\cup_{i = k+1}^\infty V_i}(A).
\end{align*}
But this is a contradiction.
Thus $\PT$ is a finite pre-measure and thus uniquely extends to $\Sigma_\T$, by \cref{lem:extension}.
\end{proof}

The actual proof is now straightforward:
\begin{proof}[\cref{prop:exist_PT}]
Let $W_1,W_2 \in \W$ with $\D_{W_1} \neq \D_{W_2}$. By \cref{lem:prop_exist:help1}, for any finite set $\{W_1,W_2\} \subset \W_0 \subset \W$ there is a unique measure $\mu_{\W_0}$ that is compatible with $(\W_0,\Null, \D_W)$ and $\mu_{\W_0}(W_1) = 1$.

Define $\PT(S) = \mu_{\{W_1,W_2,S\}}$ for $S \in \W \cup \Null$. For $S_1,\dots,S_n \in \W \cup \Null$ disjoint with $S_1 \cup \dots \cup S_n \in \W \cup \Null$ we have that
\begin{align*}
    \PT\left(\bigcup_{i = 1}^n S_i\right)
    &= \mu_{\{W_1,W_2,S_1 \cup \dots \cup S_n\}}\left(\bigcup_{i = 1}^n S_i\right)
    \\&\overset{!}{=} \mu_{\{W_1,W_2,S_1, \dots, S_n\}}\left(\bigcup_{i = 1}^n S_i\right)
    \\&= \sum_{i = 1}^n\mu_{\{W_1,W_2,S_1, \dots, S_n\}}\left(S_i\right)
    \\&\overset{!}{=} \sum_{i = 1}^n\mu_{\{W_1,W_2,S_i\}}\left(S_i\right)
    \\&= \sum_{i = 1}^n\PT\left(S_i\right),
\end{align*}
where $!$ holds because the construction in \cref{lem:prop_exist:help1} is compatible with restrictions.
Thus, $\PT$ is a content on $\W \cup \Null$ that is compatible with $\D_W$ which by \cref{lem:prop_exist:help2} uniquely extends to a measure. 
\end{proof}

\subsection{Proof of \cref{cor:premain} and \cref{thm:main}}
Both statements can be seen as direct consequences of \cref{prop:premain,prop:exist_PT}.

\begin{proof}[\cref{cor:premain}]
As the time distribution obtained in \cref{prop:exist_PT} is unique it follows that $R(i(\D))(S) = \D(\X \times S)$. 
\end{proof}

\begin{proof}[\cref{thm:main}]
As $\X$ and $\T$ are standard Borel spaces, every probability measure can uniquely be decomposed into a marginal distribution and a Markov kernel. Thus, $\holist$ is a bijection. The statement follows by concatenating \cref{prop:premain} and \cref{cor:premain} with this map.
\end{proof}
}

\end{document}